\documentclass{article}

\usepackage[preprint]{neurips_2022}

\usepackage[utf8]{inputenc} 
\usepackage[T1]{fontenc}    
\usepackage{hyperref}       
\usepackage{url}            
\usepackage{booktabs}       
\usepackage{amsfonts}       
\usepackage{nicefrac}       
\usepackage{microtype}      
\usepackage{xcolor}         

\usepackage{subcaption}
\usepackage{microtype} 
\usepackage{graphicx}  
\usepackage[english]{babel}
\usepackage{multirow}
\usepackage{algorithmic}
\usepackage[ruled,vlined]{algorithm2e}

\usepackage{amsmath}
\usepackage{amsthm}
\usepackage{amssymb}

\newtheorem{theorem}{Theorem}  

\DeclareMathOperator*{\argmin}{arg\,min}

\DeclareMathOperator*{\argminC}{\arg\min}   

\usepackage[normalem]{ulem}
\useunder{\uline}{\ul}{}

\title{ACE: Adaptive Constraint-aware Early Stopping in Hyperparameter Optimization}

\author{%
  Yi-Wei Chen \\
  Texas A\&M University\\
  \texttt{yiwei\_chen@tamu.edu} \\
   \And
   Chi Wang \\
   Microsoft Corporation \\
   \texttt{wang.chi@microsoft.com} \\
   \AND
   Amin Saied \\
   Microsoft Corporation \\
   \texttt{amin.saied@microsoft.com} \\
   \And
   Rui Zhuang \\
  Microsoft Corporation \\
   \texttt{ruzhuang@microsoft.com} \\
}

\begin{document}

\maketitle

\begin{abstract}
Deploying machine learning models requires high model quality and needs to comply with application constraints. 
That motivates hyperparameter optimization (HPO) to tune model configurations under deployment constraints. The constraints often require additional computation cost to evaluate, and training ineligible configurations can waste a large amount to tuning cost.
In this work,
we propose an Adaptive Constraint-aware Early stopping (ACE) method to incorporate constraint evaluation into trial pruning during HPO.
To minimize the overall optimization cost,
ACE estimates the cost-effective constraint evaluation interval 
based on a theoretical analysis of the expected evaluation cost.
Meanwhile, 
we propose a stratum early stopping criterion in ACE, which considers both optimization and constraint metrics in pruning and does not require regularization hyperparameters.
Our experiments demonstrate 
superior performance of ACE in hyperparameter tuning of classification tasks under fairness or under robustness constraints.
\end{abstract}

\section{Introduction}

When machine learning (ML) is deployed in real-world applications, 
practitioners desire to tune hyperparameters of a model to maximize its utility.  
The model is often required not only to optimize for ML objectives (e.g., accuracy, l2 loss, or F1 scores), but also to meet the deployment constraints, such as latency, storage, fairness, robustness, and explainablility.
For example, when fairness is concerned,
it is required to treat different demographics fairly~\citep{perrone2021fair}.
The constraints increase the challenge of hyperparameter optimization (HPO).
If constraints are difficult to meet, 
we may spend high cost on ineligible trials, which violate the constraints. 
If constraints are expensive to compute, the step of constraint checking adds non-trivial overhead to HPO. 
How to early stop ineligible trials and how frequent to evaluate constraints are important factors for constrained HPO. 
Moreover, constraints vary in different problems in terms of constraint checking
cost and difficulty to satisfy. 
A good solution needs to adapt to unknown constraint characteristics automatically.

We propose Adaptive Constraint-aware Early stopping (ACE), which makes use of constraint evaluation to terminate inferior trials. It decides the frequency to check constraints and
the moment to stop trials 
to improve the efficiency of constrained HPO.
We model the expected trial cost in terms of the constraint evaluation interval
and provide theoretical analysis for the optimal constraint evaluation interval in the cost model.
Based on the theoretical results,
ACE suggests cost-effective constraint evaluation interval for each trial. 
Meanwhile, we propose a simple yet effective \emph{stratum truncation} policy to prune unqualified trials. 
According to constraint evaluation results, 
the approach categorizes trials into groups: no-constraint-evaluation, satisfying-constraints, and violating-constraints.
It terminates inferior trials of each group, 
which have either a relatively large extent of constraint violation or bad optimization objective. 
Unlike regularization
methods~\citep{xu2017gv}, 
our approach does not introduce additional penalty hyperparameters. It automatically balances the goal of optimization and constraint satisfaction.

ACE is the first constraint-aware early stopping algorithm for HPO. 
Our experiments demonstrate that
ACE obtains superior performance to constraint-agnostic early stopping baselines~\citep{li2020system}, on UCI credit card dataset~\citep{yeh2009comparisons} with a fairness constraint~\citep{bird2020fairlearn}
and on GLUE SST2~\citep{wang2018glue} with 
a robustness constraint~\citep{ribeiro-etal-2020-beyond}.



\begin{figure}
    \centering
    \includegraphics[width=0.7\linewidth]{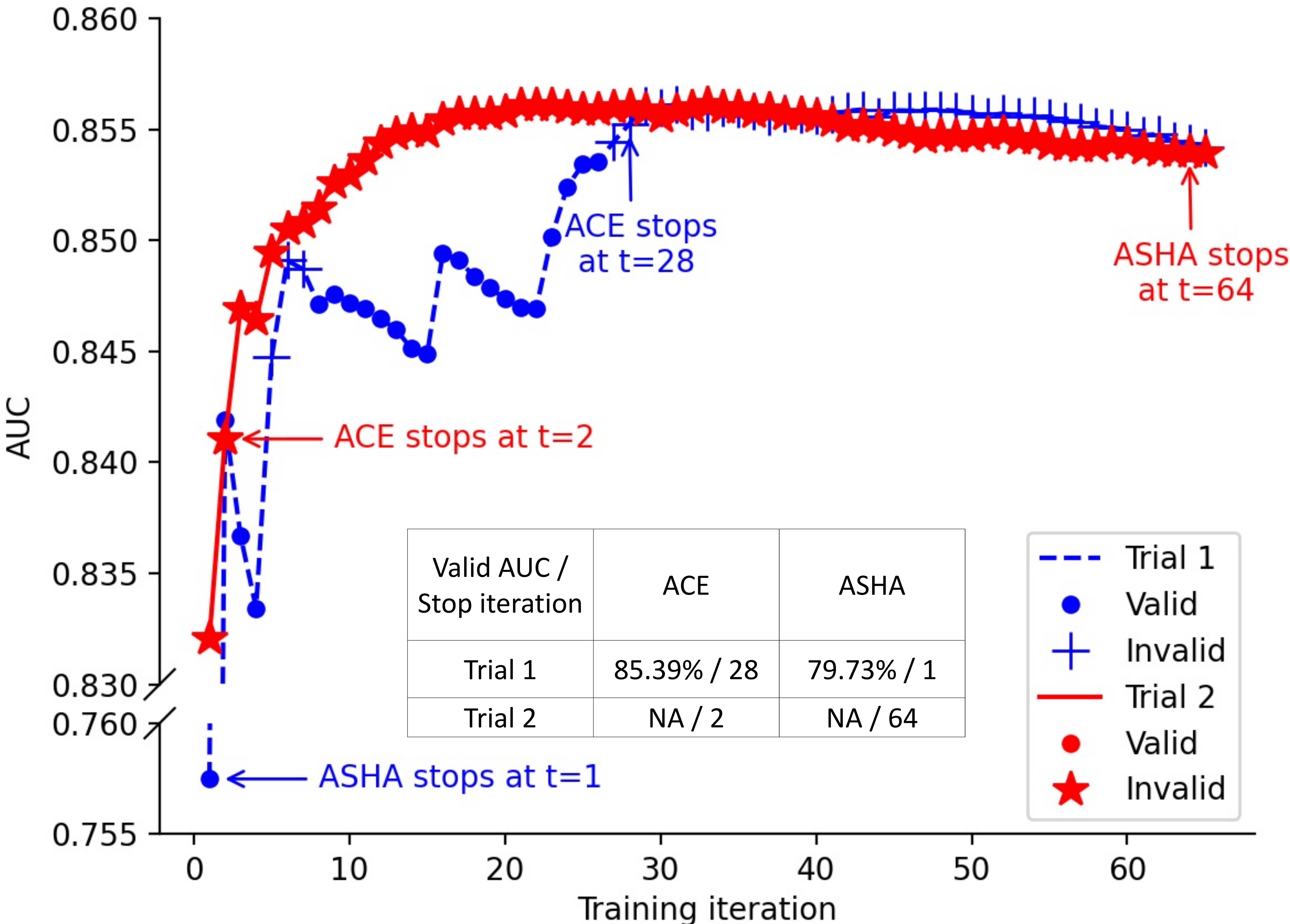}
    \caption{
    How ACE (our method) and ASHA~\citep{li2020system} terminate two trials from random search.
    \textquotesingle $\bullet$\textquotesingle\ signifies a valid checkpoint (training iteration), satisfying the constraint,
    while \textquotesingle $+$\textquotesingle\ represents an invalid checkpoint, violating the constraint.
    ACE is a constraint-aware early stopping approach,
    while ASHA is a constraint-agnostic method.
    For Trial 1,
    ACE tolerates small extent of constraint violation in checkpoints 5-7 and stops it at the 28th checkpoint, close to the optimal stop point (26) for this trial.
    ASHA wrongly stops this promising trial in the beginning.
    For Trial 2, all the checkpoints are invalid.
    ACE can stop this infeasible trial early due to large extent of constraint violation,
    while ASHA keeps training it until $t=64$.
    These are only two example trials to illustrate the potential advantage of using constraint information in early stopping and explain the overall empirical effectiveness of ACE
    even though the stop point is not necessarily optimal for every trial. 
    }
    \label{fig:when_to_prune}
\end{figure}

\section{Related work}
Constrained hyperparameter optimization (HPO) 
uses constraint information to suggest the next candidate trials.
In Bayesian optimization~\citep{gelbart2014bayesian, gardner2014bayesian, bernardo2011optimization, perrone2021fair}, 
they use additional Gaussian process to model the probability of the constraint violation.
The utility function of expected improvement (EI) is multiplied by the constraint probability to suggest the next promising and feasible trials.
The constraint violation can also function as a regularization term to penalize the optimization metric~\citep{xu2017gv}.
The penalized metrics lead the optimization algorithm toward feasible regions with less violation.
Since their penalty hyperparameter is sensitive to the scale of constraint violation,
penalized methods require extra effort to tune the penalty.
These methods check whether a trial meets or violates constraints once after a trial completes training.
When constraints are difficult to satisfy, we may spend unnecessary training cost on ineligible trials.
Our method can examine constraints during training and stop trials without extra penalty hyperparameters.

Early stopping methods are widely applied to unconstrained HPO to reduce training time of inferior trials.
Dynamic budget allocation~\citep{li2017hyperband, li2020system, jamieson2016non} can run $N$ trials for a small budget (e.g., training iterations).
They iteratively select the best \(1/c\) portion and increase their budget by \(c\) times.
Median stopping policy~\citep{golovin2017google} stops a trial if the trial's best optimization metric is worse than the median value of running average of all completed trials.
Bandit stopping policy~\citep{rasley2017hyperdrive} compares a trial's best optimization metric to
an allowable slack ratio of the global best optimization metric.
Performance curve stopping policy~\citep{swersky2014freeze, golovin2017google, domhan2015speeding} use training curves of completed trials to train a regression model,
which predict an optimization metric for an incoming trial.
If the probability of exceeding the optimal optimization metric is low,
the trial is stopped.
These early stopping methods prune trials only based on the optimization metric.
In the existing tuning frameworks~\citep{liaw2018tune, feurer-neurips15a} and business services~\citep{das2020amazon, golovin2017google},
they provide classical early stopping approaches by comparing optimization metrics.
No constraint information is used in their early stopping policy.
We are motivated to augment these services to handle constraints required by practitioners.

\section{Adaptive constraint-aware early stopping (ACE)}
\label{sec:ace}
We present the target problem of ACE in Section~\ref{subsec:problem},
followed by our model of the trial cost with respect to the constraint evaluation interval in Section~\ref{subsec:cost}.
The cost model 
is used to suggest cost-effective constraint evaluation interval for each trial (Section~\ref{subsec:interval}).
To use both optimization and constraint metrics for pruning unqualified trials,
we propose a \emph{stratum truncation} policy  (Section~\ref{subsec:stratum}). Finally, we further optimize our method to reduce cost in constraint evaluations. (Section~\ref{subsec:skip}).

\subsection{Problem statement}
\label{subsec:problem}
Given 
a constraint metric $g$ and the constraint threshold $\tau$,
and the optimization metric $\ell$, 
we target at finding the best feasible hyperparameter configuration $x \in \mathcal{X}$ with minimal computation cost. The best feasible hyperparameter configuration is defined as:
\begin{gather*}
    \argmin_{x \in \mathcal{X}} \ell(x) 
    \text{ s.t. } g(x) \le \tau,
\end{gather*}
The optimization metric could be
validation loss or other metrics to minimize.
We assume the values of the optimization metrics and constraint metrics 
change as training iterations increase (otherwise we only need to evaluate them once per trial).
We also assume the constraint metric evaluation incurs non-negligible cost, and the constraint metrics are selectively evaluated at some training iterations (i.e., checkpoints). 
\(\ell(x)\) and \(g(x)\) correspond to the metrics at the best checkpoint among all the training iterations where these metrics are evaluated. 
We focus on designing an effective early stopping policy from 
two aspects: when to evaluate constraint metrics, and when to terminate a trial.
Our early stopping policy is not tied to a particular hyperparameter search algorithm.

\subsection{Expected trial cost}
\label{subsec:cost}
For a given trial, let $C_1$ be the \emph{constraint cost}, the evaluation cost of 
the constraint metrics at one training iteration.
Let $C_2$ be the \emph{primary cost}, the training cost
plus the evaluation cost of the optimization metric for one training iteration.
We aim to reduce the total cost 
by choosing an appropriate \emph{constraint evaluation interval}, i.e., the frequency of computing the constraint metrics. When the constraint evaluation interval is \(\beta\), the constraint metrics are computed every \(\beta\) training iterations. In other words, the constraints will be evaluated at iterations
$\beta, 2\beta, \ldots, (z-1)\beta$, or $z\beta$,
where $z$ is the smallest integer such that $(z-1)\beta < T \leq z\beta$, and
$T$ is the maximal iteration of a trial.
We further assume that $p$ is the stop probability to prune a trial.
For instance,
if we stop a fixed percentage of inferior trials at each checkpoint,
the stop percentage can function as the stop probability.
Then, we can formulate the expected trial cost by:
\begin{equation}
\label{eq:e_cost}
\begin{aligned}
    \mathbb{E}[C]=(1-p)^z [C_1 z+C_2 T]+ \sum_{k=1}^{z} (1-p)^{k-1}  p[C_1 k +C_2 k\beta].
\end{aligned}
\end{equation}
Its first part models the cost if a trial does not stop early,
while the second part models the cost of terminating a trial after $k$ times of constraint evaluations, for \(k=1,2,\ldots,z\).
There are $k$ constraint evaluations in $k\beta$ training iterations.
Eq.~\eqref{eq:e_cost} is general to any early stopping policy with linear constraint checking intervals.

We define cost ratio as 
$r = C_1  / C_2 $.
Then, we can simplify Eq.~\eqref{eq:e_cost} into $\mathbb{E}[C] = C_2 (r+\beta)\frac{1-(1-p)^{z}}{p}$. 
The derivation is shown in Appendix~\ref{appendix:derivation}.
By definition, \(z \ge \frac{T}{\beta}\).
If \(z>\frac{T}{\beta}\), 
then we can set \(\beta'=\frac{T}{z}<\beta\) and decrease \(\mathbb{E}[C]\).
Thus, an optimal \(\beta\) in \([1,T]\) should make \(z=\frac{T}{\beta}\).
We only need to minimize:
\begin{equation}
\label{eq:derived_cost}
    \mathbb{E}[C
    |\beta
    ] = C_2 (r+\beta)\frac{1-(1-p)^{\frac{T}{\beta}}}{p}.
\end{equation}

\begin{figure}
  \begin{subfigure}[t]{0.5\linewidth}
    \centering
    \begin{minipage}[b]{1.0\linewidth}
        \includegraphics[width=\linewidth]{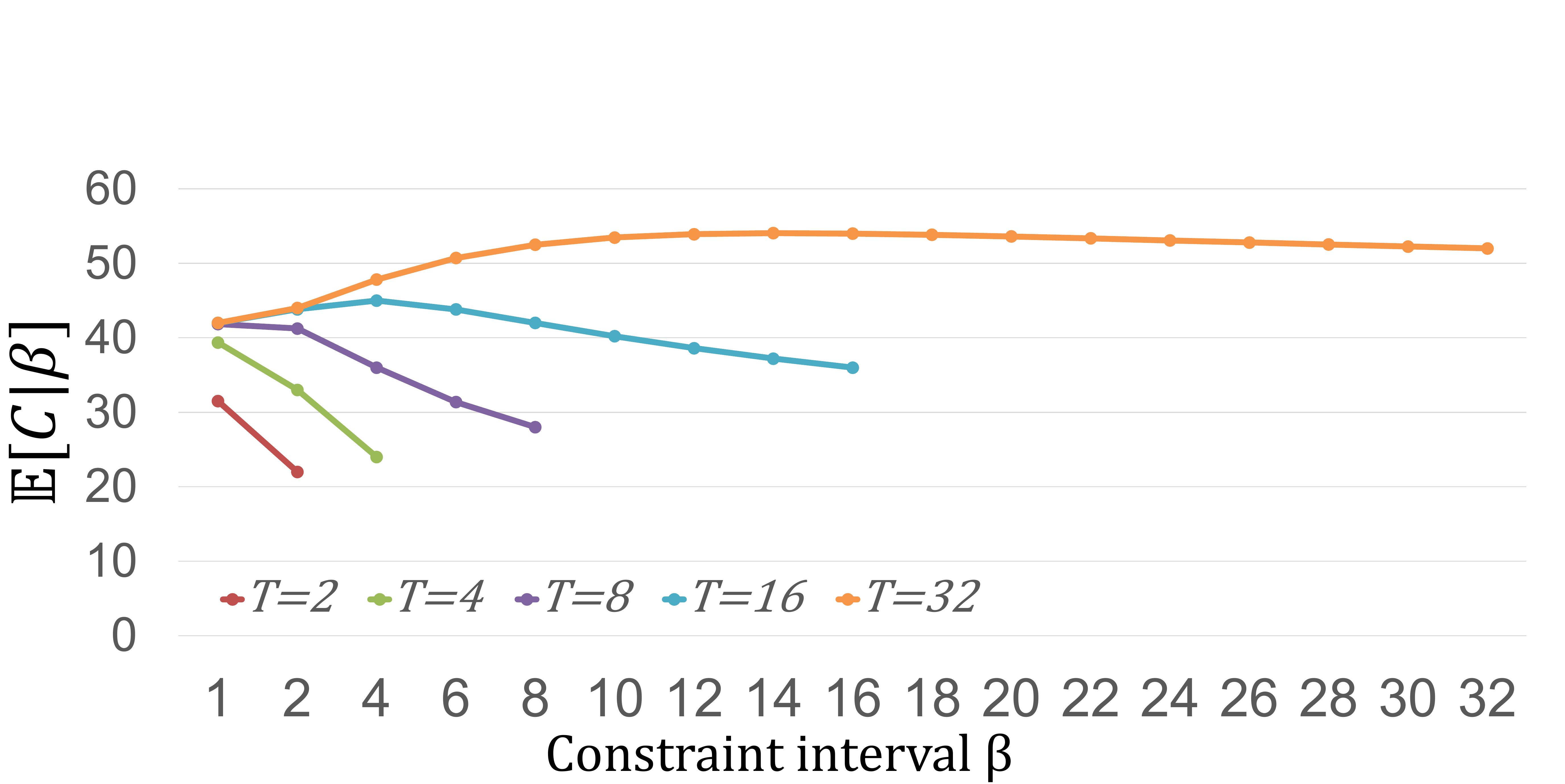}
    \end{minipage}
    \caption{$\mathbb{E}[C|\beta]$ given $r=20$ and $p=0.5$.}
    \label{fig:cost_curve_change_t}
  \end{subfigure}
  \begin{subfigure}[t]{0.5\linewidth}
    \centering
    \begin{minipage}[b]{1.0\linewidth}
        \includegraphics[width=\linewidth ]{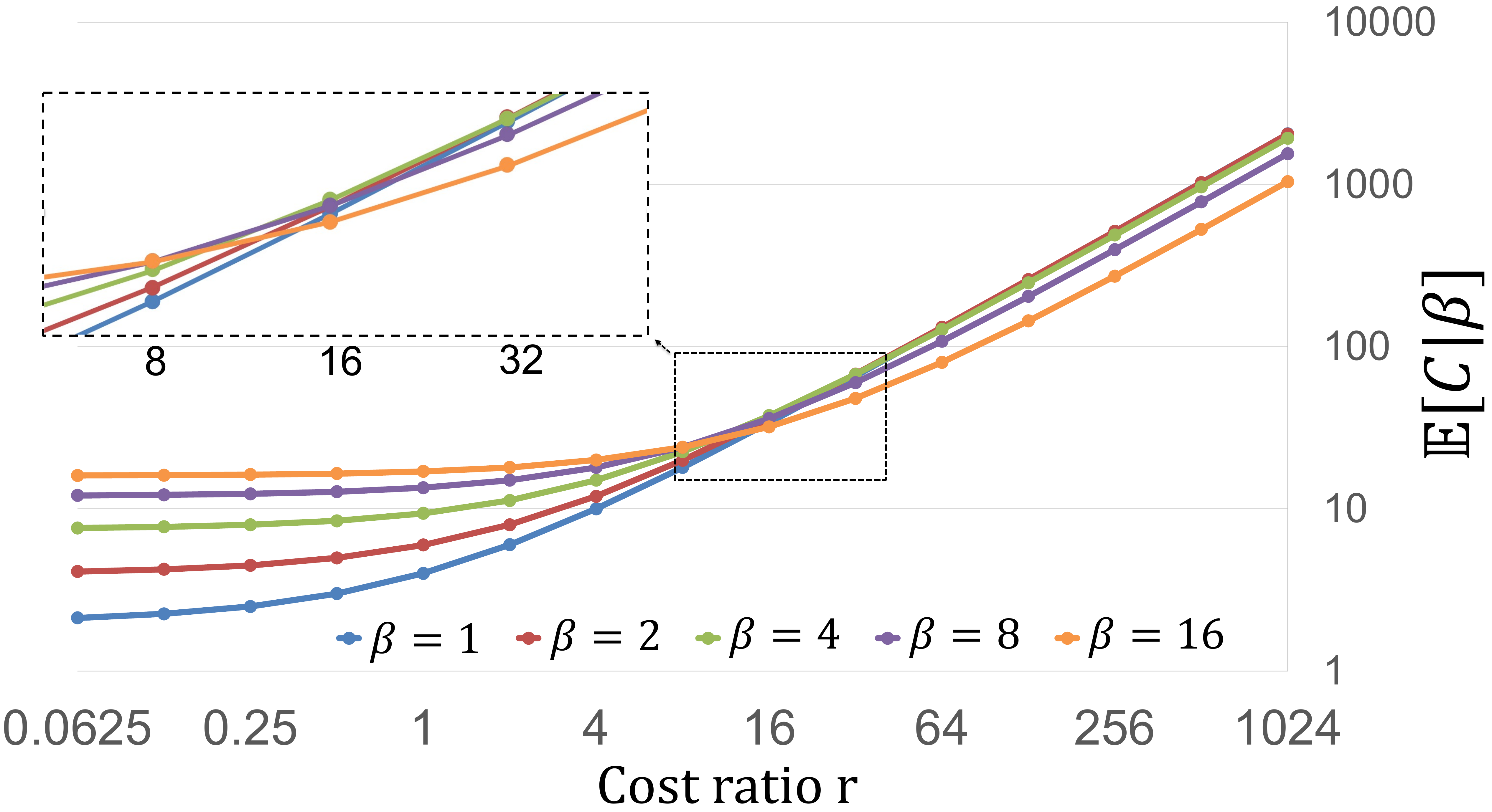}
    \end{minipage}
    \caption{$\mathbb{E}[C|\beta]$ given $p=0.5$ and $T=16$.}
    \label{fig:cost_curve_change_p}
  \end{subfigure}
  \caption{Examples of the expected trial cost in a stop probability $p=0.5$. Left:
    given the cost ratio $r=20$, 
    how the cost changes for various max iterations $T$. Right:
    given $T=16$,
    how the cost changes for different cost ratio $r$.
    }
  \label{example_figure}
\end{figure}

\subsection{Constraint evaluation interval}
\label{subsec:interval}
Given $p$, $T$, and $r$,
we want to obtain the optimal constraint evaluation interval $\beta^*$ to minimize Eq.~\eqref{eq:derived_cost}.
We analyze the equation by fixing the stop probability $p=0.5$.
First of all, 
we observe how max iteration $T$ changes $\mathbb{E}[C|\beta]$ in Figure~\ref{fig:cost_curve_change_t},
when the cost ratio $r$ is 20.
It points out that $\beta^*$ is either 1 or $T$ for different trials.
Then,
we set $T=16$ to observe how $r$ changes $\mathbb{E}[C|\beta]$ in Figure~\ref{fig:cost_curve_change_p}.
It also demonstrates that
$\beta^*$ appears in $\{1, T\}$.
The left part of the figure ($r=0.0625$ to $8$)
favors $\beta=1$,
while 
its right part ($r=16$ to 1,024)
prefers $\beta=T$.
The two cases both indicate
that the optimal $\beta^*$ is situated in $\{1, T\}$,
the two ends of $\beta$'s range.
We can prove Theorem~\ref{thrm:optimal}
for general cases:
\begin{theorem}
\label{thrm:optimal}
$\forall p, r$, and $T$,
the optimal $\beta^*$ for Eq.~\eqref{eq:derived_cost}
is 1 or \(T\), i.e., $\argminC_{\beta \in [1, T]} \mathbb{E}[C
|\beta
]\in \{1, T\}$.
If $\ r<\frac{pT+(1-p)^{T}-1}{1-p-(1-p)^{T}}, 
\beta^* = 1.$
If \(\ r>\frac{pT+(1-p)^{T}-1}{1-p-(1-p)^{T}}\), 
\(\beta^* = T\).
If \(\ r=\frac{pT+(1-p)^{T}-1}{1-p-(1-p)^{T}}\), 
\(\beta^* = \{1, T\}\).
\end{theorem}

\begin{proof}
We define $s := (1-p)^{-T} > 1$ and $T>1$.
\(f(\beta) = (r+\beta)\frac{1-s^{-\frac{1}{\beta}}}{p}\) and \(\mathbb{E}[C
|\beta
]\) have the same minimizers for \(\beta\in [1,T]\).
We take the derivative of \(f(\beta)\):
\begin{align*}
    f'(\beta) = \dfrac{1-s^{-\frac{1}{\beta}}}{p}-\dfrac{\ln\left(s\right)\left({\beta}+r\right)}{ps^\frac{1}{{\beta}}{\beta}^2}
\end{align*}
$\because s > 1$, we notice that \(f'(\beta)\) decreases as \(r\) increases, and:
\begin{align}
\label{eq:root}
    f'(\beta) = 0 \Leftrightarrow r=-\dfrac{{\beta}\ln\left(s\right)-{\beta}^2s^\frac{1}{{\beta}}+{\beta}^2}{\ln\left(s\right)} 
    = g(\beta)
\end{align}
Then we take the derivative of \(g(\beta)\), $h(s)$, and $i(s)$:
\begin{align*}
     g'(\beta) &= \dfrac{\left(2s^\frac{1}{{\beta}}-2\right){\beta}-s^\frac{1}{{\beta}}\ln\left(s\right)-\ln\left(s\right)}{\ln\left(s\right)} 
    \triangleq \frac{h(s)}{\ln{s}} \\
    h'(s) &= -\dfrac{s^\frac{1}{{\beta}}\ln\left(s\right)-{\beta}s^\frac{1}{{\beta}}+{\beta}}{{\beta}s} 
    \triangleq - \frac{i(s)}{\beta s} \\
    i'(s) &= \dfrac{s^{\frac{1}{{\beta}}-1}\ln\left(s\right)}{{\beta}} \ge 0, s\ge 1
\end{align*}
$\because i'(s) \ge 0$, \(i(s)\ge i(1)=0, s\ge 1\). 
That means \(h'(s)\le 0, s\ge 1\). So \(h(s)< h(1)=0, s>1\). Hence, \(g'(\beta)<0, s>1\).
Now, assume \(f'(\beta_0)=0\) for \(\beta_0\in [1,T]\). We know from Eq.~\eqref{eq:root} that \(r=g(\beta_0)\). For any \(1\le \beta_1 < \beta_0\), we know \(g(\beta_1)>g(\beta_0)=r\) according to \(g'(\beta)<0\). 
Because \(f'(\beta)\) decreases when \(r\) increases, we have:
\begin{align*}
    f'(\beta_1) > \dfrac{1-s^{-\frac{1}{\beta_1}}}{p}-\dfrac{\ln\left(s\right)\left({\beta_1}+g(\beta_1)\right)}{ps^\frac{1}{{\beta_1}}{\beta_1}^2} = 0
\end{align*}
Likewise, for any \(\beta< \beta_2 \le T\), $f'(\beta_2) < 0$.
These mean that if \(f'(\beta)\) has a root \(\beta_0\) in \([1,T]\), \(f(\beta)\) is monotonically increasing in \([1,\beta_0]\), and monotonically decreasing in \([\beta_0, T]\). If \(f'(\beta)\) does not have a root in \([1,T]\), then it is either monotonically increasing or monotonically decreasing in \([1, T]\) because \(f'(\beta)\) is continuous in \([1,T]\). Therefore, in either case, the minimizer of \(f(\beta)\) is either 1 or \(T\).

Finally, we compare \(f(1)\) and \(f(T)\).
\begin{align*}
    & f(1) < f(T) \Leftrightarrow (r+1)(1-s^{-1})<(r+T)(1-s^{-1/T})\\
    \Leftrightarrow & (r+1)(1-(1-p)^{T})<(r+T)p \\
    \Leftrightarrow & r<\frac{pT+(1-p)^{T}-1}{1-p-(1-p)^{T}}
\end{align*}
\end{proof}

Theorem \ref{thrm:optimal} states that
the best constraint interval for Eq.~\eqref{eq:derived_cost} is either $1$ or $T$.
It also formulates the relationship between $r$, $p$, $T$, and $\beta$.
By substituting $p=0.5$ and $T=16$,
we get \(\ \frac{pT+(1-p)^{T}-1}{1-p-(1-p)^{T}} = 14\).
Figure~\ref{fig:cost_curve_change_p} verifies the theoretical value: if cost ratios are less than 14, $\beta=1$ costs the least.
By using $p=0.5$ and $r=20$,
we get \(\ 20 > \frac{0.5T+(0.5)^{T}-1}{0.5-(0.5)^{T}} \rightarrow 22 > T\).
Figure~\ref{fig:cost_curve_change_t} also verifies the theoretical value:
curves ($T=2, 4, 8, 16$) both favor $\beta=T$ to reach the smallest expected cost.
The two figures confirm Theorem \ref{thrm:optimal}.


According to the analysis, 
we design the adaptive constraint evaluation interval as follows.
Each trial receives its own constraint evaluation interval.
In the beginning, 
we do not know the cost ratio $r$.
The worst case is that the constraint evaluation is expensive.
Using $\beta=1$ as the initial value will lead to great cost.
Thus, ACE selects $\beta=T$ as the initial value,
which means the constraint evaluation occurs at the end of training.
During the training of a trial, 
ACE records the primary cost and constraint cost on the fly.
ACE computes the cost ratio using the average constraint cost divided by the average primary cost.
After the cost ratio is computed, it 
picks the low-cost \(\beta\) as the constraint evaluation interval by comparing \(r\) and $\frac{pT+(1-p)^{T}-1}{1-p-(1-p)^{T}}$.
Line~\ref{alg:constraint_interval} in Algorithm~\ref{alg:ace} implements the above statements.


\subsection{Stratum truncation}
\label{subsec:stratum}
We propose \emph{stratum truncation} to use both optimization and constraint metrics to prune trials (Line~\ref{alg:stratum:start} to~\ref{alg:stratum:end} of Algorithm~\ref{alg:ace}).
It is motivated by several desiderata. First, we want to respect both the goal of optimization and constraint satisfaction in the early stopping. That requires pruning both the trials with bad optimization metric and the trials violating the constraints. Second, an invalid yet incomplete trial has a chance to meet the constraints with more training iterations. So it is not necessary to stop every invalid trial immediately. Third, as an HPO solution, it is preferred to avoid introducing additional hyperparameters such as penalty weights of the constraint metrics.

\begin{algorithm}[t]
    \caption{Adaptive Constraint-aware Early Stopping (ACE)}
    \label{alg:ace}
    \begin{algorithmic}[1]
        \REQUIRE{hyperparameter configuration $x$; total training iterations $T$;
        constraint threshold $\tau$; 
        running history $\mathcal{H}$; 
        truncation percentage $\mathcal{P}$}.
        \STATE
            Estimate the constraint evaluation interval $\beta$ based on Theorem~\ref{thrm:optimal} \label{alg:constraint_interval}
        \FOR{$t$ in $[T]$}
            \STATE 
                Train the model with $x$ for one iteration
                \STATE
                    Compute optimization metric $l(x, t)$ at the iteration $t$
            \IF{$t$ mod $\beta == 0$ and \(l(x, t) \le \mathcal{H}.f^*\)} \label{alg:low_overhead:start}
                \STATE
                    Compute constraint metric $g(x, t)$ at the iteration $t$
                \IF{$g(x, t) \le \tau$}
                    \STATE trial\_type = ``valid''
                    \STATE \(\mathcal{H}.f^*=l(x, t)\)
                \ELSE
                    \STATE 
                        trial\_type = `invalid''
                \ENDIF
            \ELSE
                \STATE trial\_type = ``no\_constraint''
            \ENDIF \label{alg:low_overhead:end}
            \STATE
                Update $\mathcal{H}$ 
            \STATE
                $\mathcal{H}'$ = $\mathcal{H}$.subset(trial\_type)
            \IF{trial\_type == ``invalid''} \label{alg:stratum:start}
                \STATE
                    $\mathcal{H}'$.sort(keys=[violation\_amount, optimization\_metric])
            \ELSE
                \STATE
                    $\mathcal{H}'$.sort(keys=[optimization\_metric])
            \ENDIF
            \IF{x is below $\mathcal{P}$\% of $\mathcal{H'}$}
                \STATE
                    stop training
            \ENDIF \label{alg:stratum:end}
        \ENDFOR
    \end{algorithmic}
\end{algorithm}

Our solution categorizes the trials into 3 groups
and prunes the same fraction of low-rank trials in each group:
(1) no-constraint: trials without constraint evaluation,
(2) valid: trials meeting constraints, and
(3) invalid: trials violating constraints. 
The no-constraint trials do not have constraint metrics evaluated at the current training iteration.
The only performance indicator is the optimization metric.
We rank the trials in this group by their optimization metric.
For valid trials, 
we only care about their optimization metric, so we
rank the trials in this group by their optimization metric.
For the invalid group,
we desire to penalize invalid trials by their violation amount ($g(x) - \tau$).
The trials with larger violation amount should rank low. 
Thus, we sort the trials by their violation amount and use the optimization metric to break ties.
Infeasible trials are not pruned immediately unless
their violation amounts are on top. That gives the `close to valid' trials a chance to meet the constraints later and potentially lead to better optimization results.
At the same time, even if no infeasible trials are turned to feasible in the end, the promising trials in the valid group are kept running just like the case of unconstrained early stopping. 

\subsection{Reduce constraint evaluation overhead}
\label{subsec:skip}
While constraint evaluation helps prune ineligible trials, it adds significant overhead if the constraint metric is expensive to compute.
To reduce such overhead, ACE performs constraint evaluation on promising trials only (Line~\ref{alg:low_overhead:start} to~\ref{alg:low_overhead:end} of Algorithm~\ref{alg:ace}).
Specifically,
ACE maintains the global best feasible score,
which is the best optimization metric value among the trials satisfying the constraint.
Only if the current optimization metric is better than the global best feasible score does ACE compute the constraint of the current checkpoint.
Otherwise, ACE skips the constraint evaluation for this checkpoint and marks
the trial type of the current checkpoint as ``no-constraint''. Note that bad trials in the ``no-constraint'' group still have the chance to be pruned due to our stratum truncation policy, even without constraint evaluation.
With this low-overhead mechanism,
ACE wastes less time in evaluating suboptimal trials.




\section{Experimental evaluation}
\label{sec:exp}
We evaluate our algorithm under a fairness constraint
and a robustness constraint respectively.
Since ACE is an early stopping approach,
our baselines are no-stopping and a state-of-the-art early stopping method, ASHA~\citep{li2020system}.
ASHA is a parallel early stopping method using successive halving~\citep{jamieson2016non} to allocate budgets for trials.
ACE uses 25\% as the truncation percentage
for all the experiments.
We analyze the impact of the truncation percentage in Appendix~\ref{appendix:truncation}.
The reported results are the average of three random seeds in all the experiments.

\begin{table}
\centering
\caption{Best feasible AUC under the fairness constraint.}
\scalebox{0.75}{
\begin{tabular}{@{}llllll@{}}
\toprule
Constraint                                                                                    & Searcher                       & Early stopping policy                                                                                           & Total trials & Best feasible AUC (\%) & Time to best \\ \midrule
\multirow{8}{*}{\begin{tabular}[c]{@{}l@{}}EOD $\le 0.25$ \\ (hard) \end{tabular}}   & \multirow{5}{*}{Blend search}  & No-stopping                                                                                        & 276.0        & $84.95 \pm 0.43$            & 0h 5m 5s    \\
                                                                                             &                                & No-stopping  w/ constraint callback                                                                & 626.3        & $84.98 \pm 0.51$            & 1h 1m 7s     \\
                                                                                             &                                & ASHA                                                                                                & 3,664.3      & $82.88 \pm 3.13$            & 1h 2m 45s    \\
                                                                                             &                                & ASHA w/ constraint callback                                                                         & 4,252.0      & $85.06 \pm 0.35$            & 0h 5m 19s    \\
                                                                                             &                                & ACE                         
                                                                                                               & 1,799.3      & \textbf{85.53 $\pm$ 0.03}   & 0h 29m 09s   \\ \cmidrule(l){2-6} 
                                                                                             & \multirow{3}{*}{Random search} & No-stopping                                                                                         & 37.0         & 84.28 $\pm$ 0.45            & 1h 0m 15s    \\
                                                                                             &                                & ASHA                                                                                                & 3,497.7      & 84.19 $\pm$ 0.35            & 1h 0m 53s    \\
                                                                                             &                                & ACE
                                                                                                               & 283.7        & \textbf{85.38 $\pm$ 0.06}   & 0h 39m 09s   \\ \midrule
\multirow{8}{*}{\begin{tabular}[c]{@{}l@{}}EOD $\le 0.325$ \\ (easy) \end{tabular}} & \multirow{5}{*}{Blend search}  & No-stopping                                                                                         & 694.3        & 85.85 $\pm$ 0.03            & 1h 0m 4s     \\
                                                                                             &                                & No-stopping   w/ constraint callback                                                                & 307.0        & $85.83 \pm 0.05$            & 0h 15m 57s     \\
                                                                                             &                                & ASHA                                                                                                & 3,763.3      & $85.82 \pm 0.08$            & 1h 0m 18s    \\
                                                                                             &                                & ASHA w/ constraint callback                                                                         & 3,677.3      & $85.81 \pm 0.11$            & 0h 24m 15s    \\
                                                                                             &                                & ACE
                                                                                                               & 2,197.7      & \textbf{85.89 $\pm$ 0.00}   & 0h 54m 22s   \\ \cmidrule(l){2-6} 
                                                                                             & \multirow{3}{*}{Random search} & No-stopping                                                                                         & 37.0         & 85.71 $\pm$ 0.10            & 1h 0m 03s    \\
                                                                                             &                                & ASHA                                                                                                & 3,596.7      & 85.77 $\pm$ 0.03      & 1h 0m 17s    \\
                                                                                             &                                & ACE
                                                                                                               & 1,200.0      & \textbf{85.79 $\pm$ 0.07}   & 0h 19m 37s   \\ \bottomrule
\end{tabular}
}
\label{tab:fair}
\end{table}

\subsection{Fairness constraint}
\label{subsec:fairness}
We follow Fairlearn~\citep{bird2020fairlearn} to preprocess the UCI credit card default dataset~\citep{yeh2009comparisons}.
The processed dataset has 30k clients with 22 features,
which we split 70/30\% for training/validation.
The fairness constraint is Equalized-Odd-Difference (EOD)~\citep{bird2020fairlearn} 
= $\max(|FPR_1 - FPR_2|, |FNR_1 - FNR_2|)$,
where $FPR$ and $FNR$ mean false positive rate and false negative rate, respectively.
EOD quantifies the accuracy disparity experienced by different groups.
We search hyperparameters of LightGBMs under two threshold levels:
(i) the hard constraint: $EOD \leq 0.25$ and
(ii) the easy constraint: $EOD \leq 0.325$.
For each case, we evaluate ACE's performance 
with 
unconstrained random search~\citep{bergstra2012random} 
or constrained blend search~\citep{wang2021blendsearch}.
Blend search combines global and local search and prioritizes their suggestion on the fly.
Its implementation in the FLAML library~\citep{wang2021flaml} can handle constrained optimization.
It penalizes the optimization metric with the amount of constraint violation.
Since the constrained HPO needs constraint information,
we implement the constraint callbacks for the baselines.
The callbacks compute constraint metrics for a searcher.
The search budget is one hour.
The optimization metric is Area Under the Curve (AUC) of ROC~\citep{davis2006relationship}.
Appendix~\ref{appendix:fairness_setup} summarizes the rest of experiment setup.

Table~\ref{tab:fair} shows the results under the fairness constraint.
``Time to best'' means the time to find the best feasible trial within each run from the beginning of search.
For baselines without the constraint callback, 
after the tuning,
we sort the searched trials by their optimization metric and compute constraints from top to down until we find the first feasible trial. So their "time to best" is longer than the one hour budget. If a method can find the highest feasible AUC compared to other methods, then smaller ``Time to best'' is better. Otherwise, smaller ``Time to best'' just indicates a method converges to a suboptimal point early.

ACE outperforms the baselines in terms of feasible AUC when using both searchers under both easy and hard constraints.
In blend search with the hard constraint (top Table~\ref{tab:fair}), 
ACE needs only 1/2 of the search budget to attain the best feasible AUC.
ASHA and no-stopping converge to worse feasible AUC even when using the full budget. Their feasible AUC are not improved after 5 minutes.
Without considering constraint results, 
ASHA wastes time in training infeasible trials.
Furthermore,
although random search is an unconstrained searcher,
ACE 
improves feasible AUC by 1.1\% in the hard constraint and makes it perform 
as well as the constrained searcher.
We make a case study of two trials both appearing in ACE and ASHA, shown in Figure~\ref{fig:when_to_prune}.
ASHA wrongly prunes a promising trial (blue curve) very quickly
and spends unnecessary training time on an infeasible trial (red curve).
ACE successfully prunes them at the appropriate iteration.
It highlights that adding the constraint violation in the performance ranking makes ACE allocate the training budget to promising trials.
For the easy constraint, 
constraint violations are rare.
Our constrained early stopping gets close to constraint-agnostic early stopping,
as few trials are categorized into the ``invalid'' group.
In this case,
ACE still performs as well as ASHA (bottom Table~\ref{tab:fair}).
It implies our stopping policy can also work well with limited utility of constraint feedback.

\subsection{Robustness constraint}
\label{subsec:robustness}
We use DistilBERT~\citep{sanh2019distilbert} to predict the sentiment of a given sentence from the Stanford Sentiment Treebank (SST2) of GLUE~\citep{wang2018glue}.
DistilBERT is a small and fast transformer model,
which uses BERT~\citep{kenton2019bert} as a teacher and is pretrained on the same corpus as BERT in a self-supervised fashion.
SST is split into 67,349 and 872 sentences for training and validation respectively,
where the sentences are encoded by the DistilBERT's pretrained tokenizer.
The robustness constraints come from the NLP model CHECKLIST~\citep{ribeiro-etal-2020-beyond},
which provides a sentiment test suite
to check the prediction invariance in the presence of certain perturbation
or certain corner cases.
CHECKLIST measures the test score by the failure rate, a portion of predictions failing to remain original answers with perturbation.
We use 21 test tasks in the test suite, where each task has 300 test cases.
Average failure rate (AFR) is chosen to aggregate the results of the 21 test tasks.
We use random search to find the hyperparameters for finetuning DistilBERT 
under four threshold levels: AFR $\le 0.19$, AFR $\le 0.20$, AFR $\le 0.24$, and AFR $\le 0.25$.
We report average accuracy of three random seeds for each threshold.
The rest of experiment setup are summarized in Appendix~\ref{appendix:robustness_setup}.

\begin{figure}[h]
    \centering
    \includegraphics[width=0.67\linewidth]{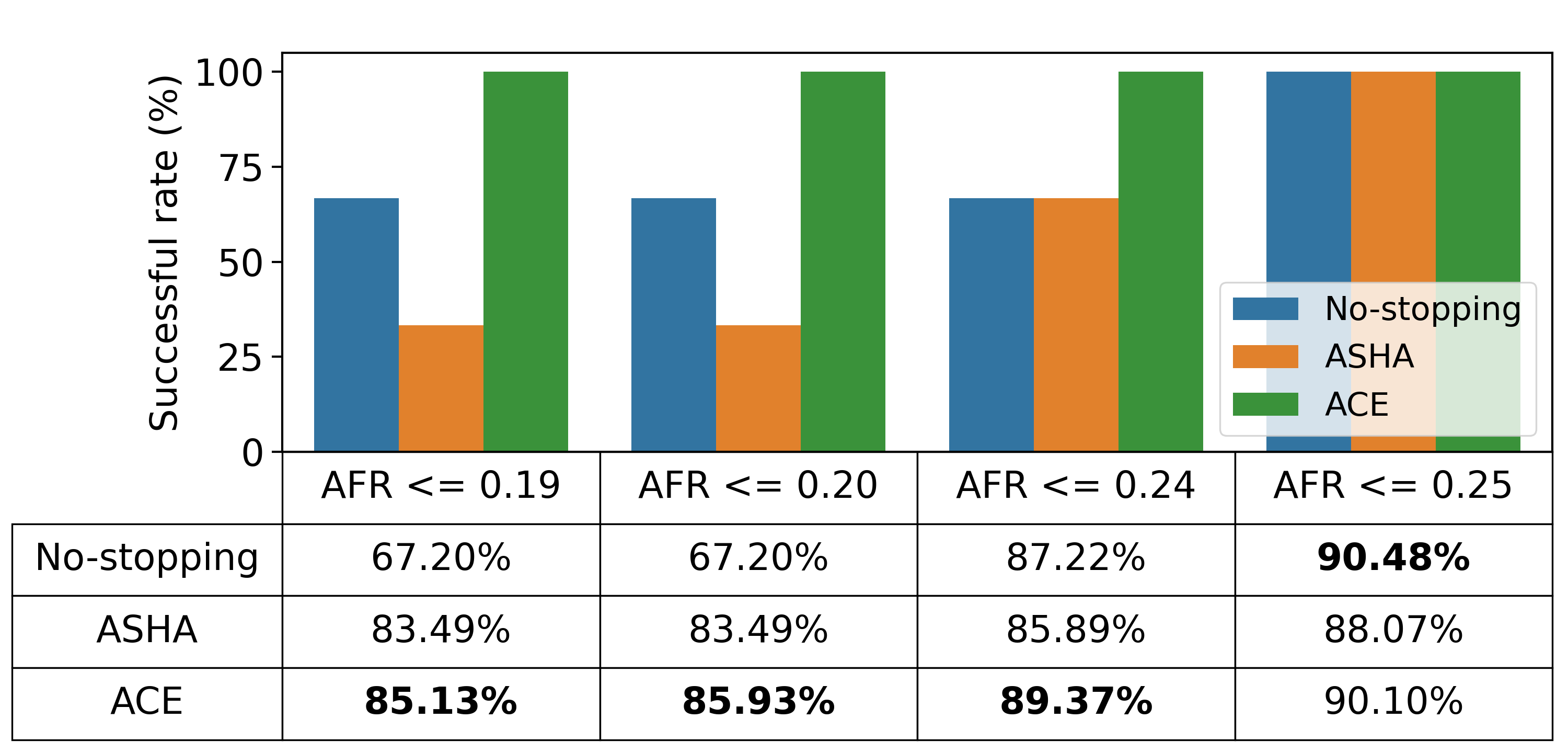}
    \caption{Best feasible accuracy under the robustness constraint.
    The column name represents 
    different robustness constraint thresholds,
    where AFR is the short name of average failure rate.
    The values in the table are the best feasible accuracy (\%).
    Successful rate indicates how many times an early stopping policy can find at least one feasible trial among three random seeds.
    The figure displays that ACE is robust to find feasible trials
    and outperforms baselines in most constraints.
    }
    \label{fig:robustness}
\end{figure}

Figure~\ref{fig:robustness} exhibits that ACE achieves 100\% successful rate to get feasible trials under all the robustness constraints,
while no-stopping and ASHA only attain 100\% successful rate under AFR $\le 0.25$.
For the best feasible accuracy (the table in Figure~\ref{fig:robustness}),
ACE obviously outperforms state-of-the-art ASHA under all the robustness constraints.
With the consideration of constraint values,
ACE can wisely invest time on feasible trials.
ACE also beats no-stopping under the first three constraint thresholds 
and performs competitively under AFR $\le 0.25$.
The best feasible trials of no-stopping under AFR $\le 0.25$
possesses the worst accuracy in the initial training iterations.
Hence, ACE prunes these two trials accordingly, as well as ASHA does.
The results reveal that reducing training overhead on inferior trials sometimes leads to wrongful termination of promising candidates.
Addressing this limitation would be promising future work.

\subsection{Ablation study}
\label{subsec:ablation}
We analyze ACE design choices by evaluating (i) different stopping criterion and (ii) static/adaptive constraint evaluation intervals in Table~\ref{tab:ablation_stopping} and Table~\ref{tab:ablation_interval}, respectively.
Since random search generates the identical sequences of hyperparameters given the same random seed,
we select it to compare different choices fairly for the ablation study.

\begin{table}[h]
\centering
\caption{Ablation study for early stopping choices.}
\scalebox{0.8}{


\begin{tabular}{@{}lllll@{}}
\toprule
Constraint                                                                       & Method       & Stopping criterion & Low-overhead evaluation & Best feasible AUC (\%)    \\ \midrule
\multirow{3}{*}{\begin{tabular}[c]{@{}l@{}}EOD $\le$ 0.25\\ (hard)\end{tabular}} & ACE$_{hard}$ & hard stopping      & $\checkmark$            & 84.94 $\pm$ 0.31          \\
                                                                                 & ACE$_{noskip}$ & stratum            &                         & 85.37 $\pm$ 0.02          \\
                                                                                 & ACE          & stratum            & $\checkmark$            & \textbf{85.38} $\pm$ 0.06 \\ \midrule
\multirow{3}{*}{\begin{tabular}[c]{@{}l@{}}EOD $\le$ 0.325\\ (easy)\end{tabular}} & ACE$_{hard}$ & hard stopping      & $\checkmark$            & 85.74 $\pm$ 0.05          \\
                                                                                 & ACE$_{noskip}$ & stratum            &                         & 85.69 $\pm$ 0.03          \\
                                                                                 & ACE          & stratum            & $\checkmark$            & \textbf{85.79} $\pm$ 0.07 \\ \bottomrule
\end{tabular}

}
\label{tab:ablation_stopping}
\end{table}

Our \emph{stratum truncation} in ACE is a soft stopping approach,
which tolerates invalid checkpoints.
A hard stopping policy terminates a trial immediately 
as soon as the trial encounters the first invalid checkpoint.
We implement the hard stopping in ACE$_{hard}$.
Table~\ref{tab:ablation_stopping} shows that
ACE$_{hard}$ decreases performance for both easy and hard fairness constraints.
The constraint values do not increase or decrease monotonically.
It might temporarily violate the constraint and meet the constraint late,
such as the blue curve in Figure~\ref{fig:when_to_prune}.
Just considering constraint validity,
ACE$_{hard}$ might terminate trials at wrong iterations.
Furthermore,
ACE$_{noskip}$ faithfully evaluates constraints without skipping suboptimal trials
which have inferior optimization metrics compared to the best feasible trial.
Table~\ref{tab:ablation_stopping} indicates that the reduction of constraint evaluation overhead improves ACE performance for the easy constraint and
does not degrade the performance for the hard constraint.
ACE$_{noskip}$ suffers from unnecessary constraint cost on suboptimal trials,
especially when the constraint is expensive to evaluate.

\begin{table}[h]
\centering
\caption{Ablation study for the adaptive constraint evaluation interval.}
\scalebox{0.8}{

\begin{tabular}{@{}llllll@{}}
\toprule
Constraint                                                                             & Cost ratio             & Method          & $\beta=1$ (\%) & $\beta=T$ (\%) & Best feasible score (\%)  \\ \midrule
\multirow{3}{*}{\begin{tabular}[c]{@{}l@{}}EOD $\le$ 0.25\\ (cheap)\end{tabular}}   & \multirow{3}{*}{1.94}  & ACE$_{\beta=1}$ & 100            & 0                               & \textbf{85.38 $\pm$ 0.05} \\
                                                                                       &                        & ACE$_{\beta=T}$ & 0              & 100                             & 84.61 $\pm$ 0.57          \\
                                                                                       &                        & ACE             & 87          & 13                           & \textbf{85.38 $\pm$ 0.06}    \\ \midrule
\multirow{3}{*}{\begin{tabular}[c]{@{}l@{}}AFR $\le$ 0.2\\ (expensive)\end{tabular}} & \multirow{3}{*}{23.98} & ACE$_{\beta=1}$ & 100            & 0                               & 84.33 $\pm$ 1.80          \\
                                                                                       &                        & ACE$_{\beta=T}$ & 0              & 100                             & 79.82 $\pm$ 2.93          \\
                                                                                       &                        & ACE             & 36             & 64                              & \textbf{85.93 $\pm$ 1.15} \\ \bottomrule
\end{tabular}

}
\label{tab:ablation_interval}
\end{table}

Next, 
we study how constraint intervals affect the performance
by replacing ACE's adaptive constraint interval with static values.
ACE$_{\beta=T}$ enforces all the trials to check constraints once at the end of the training iteration.
A trial's constraint metric is computed for
its best checkpoint (with the best optimization metric).
ACE$_{\beta=1}$ fixes $\beta=1$ for all the trials.
Table~\ref{tab:ablation_interval} demonstrates that
ACE can adjust the constraint interval for different constraint characteristics.
For the cheap constraint (low cost ratio),
ACE assigns 87\% of trials $\beta=1$.
For the expensive constraint (high cost ratio),
ACE assigns 64\% of trials $\beta=T$.
The results match the theoretical prediction
in Theorem~\ref{thrm:optimal}.
Since Eq.~\eqref{eq:derived_cost} also depends on a trial's total training iteration $T$,
$\beta$ is not fixed for all the trials.
It is not surprising that ACE$_{\beta=1}$ can achieve competitive performance of ACE
for the cheap constraint but weaker for the expensive constraint, as our theory predicts.
ACE$_{\beta=T}$ has a chance in theory to outperform ACE for the expensive constraint, 
since ACE$_{\beta=T}$ largely reduces the constraint evaluation overhead.
ACE$_{\beta=T}$ indeed searches 1.8 times more trials than ACE and 3.9 times more than ACE$_{\beta=1}$.
Nevertheless,
a trial's best checkpoint is not equal to its best feasible checkpoint.
Some feasible trials in ACE are considered as infeasible by ACE$_{\beta=T}$ for that reason.
Thus, 
ACE$_{\beta=T}$ ends up performing poorly.
ACE selects $\beta=1$ for trials with large $T$ 
(See Figure~\ref{fig:cost_curve_change_t}).
Even though ACE selects $\beta=T$ for a trial with small $T$,
its small number of checkpoints reduces the chance for ACE to use wrong checkpoint.
ACE's adaptive constraint interval is merely derived from the constraint cost consideration.
It is an interesting question for future work whether it is beneficial to adjust the constraint evaluation interval according to the feasibility consideration.

\section{Conclusion}
\label{sec:conclusion}
We study the problem of effective early stopping for constrained HPO, which saves cost in
training infeasible trials and explores other promising and feasible candidates.
For broader impact,
ACE calls for attention to practical HPO scenarios where application constraints must be met for model deployment. 
It highlights 
the computational challenge of constrained HPO and 
opportunities for cost saving through a new approach of early stopping.
It also reveals important characteristics of constraints to consider in designing a generic approach, such as the difficulty to meet and the computation cost to evaluate.
It provides a new point of research view toward constrained HPO.
ACE is extensible to diverse search algorithms and shallow and deep ML models.
It is implemented on top of RAY,
a hyperparameter tuning framework,
which supports 14 search techniques and
integrates a variety of ML frameworks, including PyTorch, Tensorflow, HuggingFace, LightGBM, etc.
Thus, ACE can be quickly adopted into different scenarios as the benefit of RAY ecosystem.
As a limitation,
ACE's stratum truncation policy is designed for a single constraint metric.
Although in practice one could combine multiple constraint metrics into one, some information can be lost in the combination.
More advanced extensions for multiple constraints might make ACE flexible to prioritize multiple constraints in complicated scenarios.




\newpage
\bibliography{reference}

\newpage
\appendix



\section{Derivation of Equation~\eqref{eq:derived_cost}}
\label{appendix:derivation}
The first part of Eq.~\eqref{eq:e_cost} is simplified as:
\begin{align*}
     (1-p)^z [C_1z+C_2T]
    & = (1-p)^{z} [rC_2z+C_2z\beta] \\
    & = (1-p)^z C_2(r + \beta) z.
\end{align*}
The second part of Eq.~\eqref{eq:e_cost} can be simplified as:
\begin{align*}
    & \sum_{k=1}^{z} (1-p)^{k-1}  p[C_1 k+C_2  k\beta]
     = \sum_{k=1}^{z} (1-p)^{k-1}  p[rC_2 k+C_2  k\beta] \\
     = & p C_2(r + \beta) \sum_{k=1}^{z} k(1-p)^{k-1}\\
    \because & \sum_{k=1}^{z} k(1-p)^{k-1} = \sum_{k=0}^{z} k(1-p)^{k-1} = \frac{d}{dp}[-\sum_{k=0}^{z}(1-p)^k] \\
    = & \frac{d}{dp}[-\frac{1-(1-p)^{z+1}}{p}] = \frac{1}{p^2} + \frac{(z+1)(1-p)^z(-1)p - (1-p)^{z+1}}{p^2} \\
    = & \frac{1 - (1-p)^z(zp + 1)}{p^2} \\
    \therefore & \sum_{k=1}^{z} (1-p)^{k-1}  p[C_1 k+C_2  k\beta] 
    = p C_2(r + \beta)\frac{1-(1-p)^z(zp + 1)}{p^2}.
\end{align*}
Combining the above equations, we get Eq.~\eqref{eq:derived_cost}.
\begin{align*}
    \mathbb{E}[C] &= (1-p)^z C_2(r + \beta) z + p C_2(r + \beta)\frac{1-(1-p)^z(zp + 1)}{p^2} \\
    &= C_2(r + \beta) \frac{zp(1-p)^z + 1 - (1-p)^z (zp + 1)}{p} \\
    &= C_2 (r+\beta)\frac{1-(1-p)^{z}}{p}.
\end{align*}

\begin{figure}[h]
    \centering
    \includegraphics[width=\linewidth]{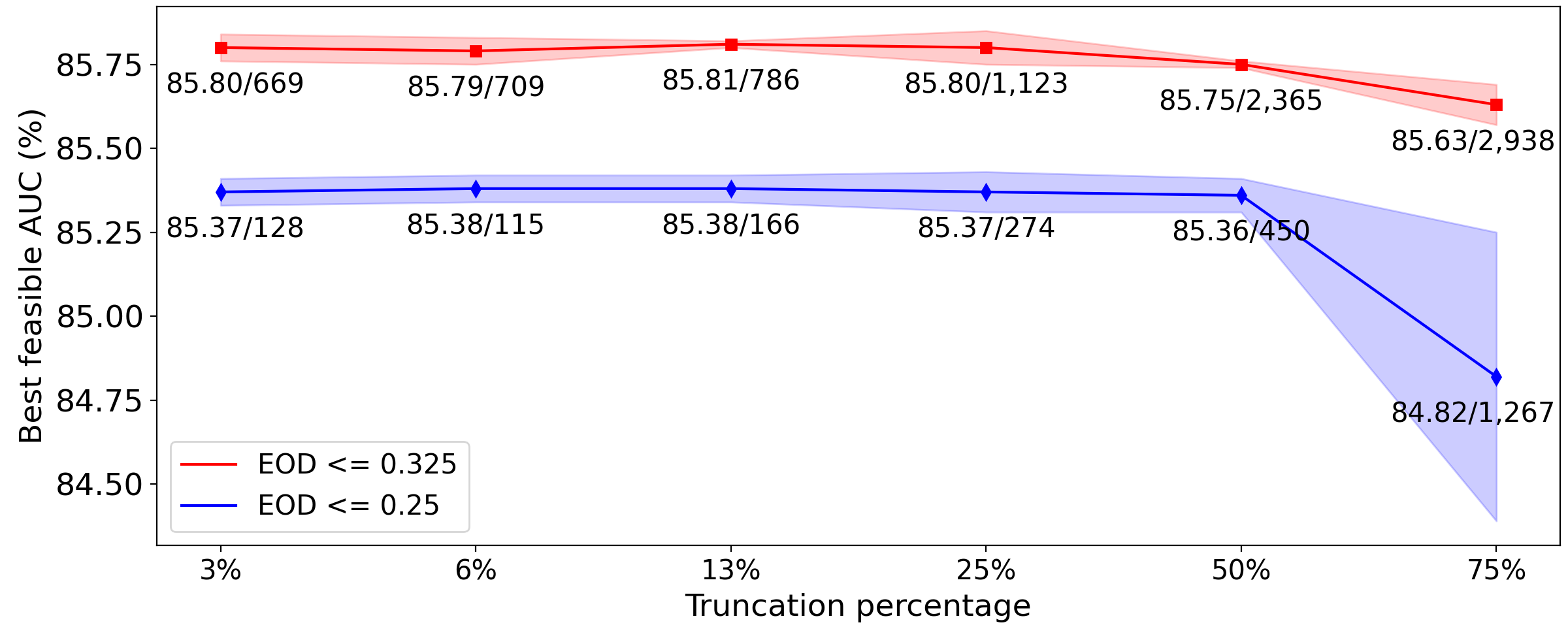}
    \caption{Truncation percentage analysis of ACE's stratum truncation. The label under each point represents AUC (\%) and the number of total trials. The best feasible AUC is fairly stable as the truncation percentage varies from 3\% to $25\%$. 
    }
    \label{fig:truncation_percentage}
\end{figure}

\section{Truncation percentage analysis}
\label{appendix:truncation}
ACE's stratum truncation terminates a fraction of low-rank trials according to the ``truncation percentage'' (See Section~\ref{subsec:stratum}). 
We analyze the impact of the percentage by searching hyperparameters of LightGBM under the fairness constraints (See Section~\ref{subsec:fairness}).
Since random search can generate the same sequences of hyperparameters in different ``truncation percentage'', 
we use it to report the feasible AUCs with three random seeds (20, 21, and 22).
Figure~\ref{fig:truncation_percentage} exhibits that 
13\% as truncation percentage leads ACE to perform the best for both hard ($EOD \le 0.25$) and easy ($EOD \le 0.325$) constraints.
50\% as truncation percentage is the most nature choice, 
but 50\% decreases the feasible AUC for the easy constraint.
The number of total trials increases from 3\% to $75\%$.
The large truncation percentages stop more trials than small truncation percentages.
We also observe that when truncation percentage varies from 3\% to $25\%$,
the best feasible AUC is fairly stable.
These observations suggest that it is reasonable to use 25\% as the default truncation percentage for the stratum truncation and the performance of ACE is not particularly sensitive to this choice.

\section{Fairness experiment setup}
\label{appendix:fairness_setup}
The fairness experiments of Section~\ref{subsec:fairness} are run on 12 cores of Intel i7-6850K@3.6 GHz.
We search hyperparameters of LightGBM under a fairness constraint,
where its search space follows FLAML~\citep{wang2021flaml},
shown in Table~\ref{tab:lgbm}.
The types in the table follow RAY~\citep{moritz2018ray}.
ASHA follows the default values of RAY library~\citep{moritz2018ray}:
reduction factor = 4, brackets = 1, grace period =1,
and max time units = 21,000 (the number of training data).
The random seeds of the experiments are 20, 21, 22.
The max concurrent trials are 4.
The fairness experiments of in Section~\ref{subsec:ablation} also follows the same setup here.

\begin{table}[htbp]
\centering
\caption{The hyperparameter space of LightGBM}
\begin{tabular}{@{}ccccc@{}}
\toprule
Name                & Type                & Lower value  & Upper value                   & Initial value \\ \midrule
n\_estimators       & ray.tune.lograndint & 4      & the number of instances & 4          \\ 
num\_leaves         & ray.tune.lograndint & 4      & the number of instances & 4          \\ 
min\_child\_samples & ray.tune.lograndint & 2      & 129                     & 20         \\
learning\_rate      & ray.tune.loguniform & 1/1024 & 1.0                     & 0.1        \\ 
log\_max\_bin       & ray.tune.lograndint & 3      & 11                      & 8          \\ 
colsample\_bytree   & ray.tune.uniform    & 0.01   & 1.0                     & 1.0        \\ 
reg\_alpha          & ray.tune.loguniform & 1/1024 & 1024                    & 1/1024     \\ 
reg\_lambda         & ray.tune.loguniform & 1/1024 & 1024                    & 1.0        
\\ \bottomrule
\end{tabular}
\label{tab:lgbm}
\end{table}

\section{Robustness experiment setup}
\label{appendix:robustness_setup}
The robustness experiments of Section~\ref{subsec:robustness} are run on 64 cores of AMD EPYC 7282 and 8 NVIDIA RTX A5000.
PYTHON 3.6 environment includes
FLAML 0.6.9, RAY 1.10, CHECKLIST 0.0.11, and HUGGINGFACE-HUB 0.4.0.
We search hyperparameters for finetuning DistilBERT on SST2
under a robustness constraint. 
The search space is motivated by FLAML library~\citep{wang2021flaml},
as show in Table~\ref{tab:distilbert}.
The types in the table follow RAY~\citep{moritz2018ray}.
The twenty one test task in the robustness constraint are illustrated in Table~\ref{tab:test_tasks}.
ASHA follows the default values of RAY library~\citep{moritz2018ray}:
reduction factor = 4, brackets = 1, grace period =1, 
and max time units = 67,349 (the number of training data).
The random seeds of the experiments are 19, 20, 21.
Th max concurrency trials are 4.
The search budget is three hours.
The robustness experiments of in Section~\ref{subsec:ablation} also follows the same setup here.

\begin{table}[htbp]
\centering
\caption{The hyperparameter space of finetuning DistilBERT}
\begin{tabular}{@{}ccccc@{}}
\toprule
Name                            & Type                & Lower value        & Upper value       & Initial value \\ \midrule
learning\_rate                  & ray.tune.loguniform & 1e-6               & 1e-3              & 1e-5          \\
num\_train\_epochs              & ray.tune.loguniform & 0.1                & 10.0              & 3.0           \\
per\_device\_train\_batch\_size & ray.tune.choice     & \multicolumn{2}{l}{4, 8, 16, 32}       & 32            \\
warmup\_ratio                   & ray.tune.uniform    & 0.0                & 0.3               & 0.0           \\
weight\_decay                   & ray.tune.uniform    & 0.0                & 0.3               & 0.0           \\
adam\_epsilon                   & ray.tune.loguniform & 1e-8               & 1e-6              & 1e-6          \\
seed                            & ray.tune.choice     & \multicolumn{2}{l}{40, 41, 42, 43, 44} & 42            \\ \bottomrule
\end{tabular}
\label{tab:distilbert}
\end{table}

\section{Comparison between geometric and linear interval}
Given the reduction factor = 4, 
ASHA examines trial performance by geometric intervals, i.e., 1, 4, 16, 64, $\dots$.
In contrast, ACE check trials by linear intervals, i.e., 1, 2, 3, $...$.
The different interval motivates us to study the performance of geometric intervals in constrained early stopping.
We extend ASHA with adaptive constraint interval and stratum truncation.
If the adaptive constraint interval suggests $\beta=1$,
we enforce ASHA to compute constraint values by geometric intervals.
If the adaptive constraint interval suggests $\beta=T$,
ASHA uses a trial's best checkpoint to evaluate constraint values at the last training iteration. 
Notice that Thereom~\ref{thrm:optimal} is developed by the linear interval assumption.
We merely borrow it to suggest constraint intervals, rather than develop a new theorem.
We use blend and random search to search hyperparameters of LightGBM under the hard fairness constraint (EOD $\le 0.25$).
In Table~\ref{tab:ablation_geometric},
we observe the stratum truncation and adaptive constraint interval can improve constraint-agnostic ASHA.
Since the constrained ASHAs do not have the low-overhead constraint evaluation,
we report ACE$_{noskip}$ to compare their performance.
ACE$_{noskip}$ outperforms ASHA$_{stratum}$ and ASHA$_{stratum\_not\_fixed}$.
A trial's best checkpoint is not equal to its best feasible checkpoint.
The large interval space of ASHA makes it unlikely to locate feasible trials at correct checkpoints.
Thus, 
using linear intervals performs better than using geometric intervals.

\begin{table}[h]
\centering
\caption{Geometric vs. linear interval}
\scalebox{0.8}{

\begin{tabular}{@{}llllll@{}}
\toprule
Searcher                       & Method           & Interval type & Stopping criterion  & Constraint interval & Best feasible AUC (\%)    \\ \midrule
\multirow{4}{*}{Blend search}  & ASHA             & geometric     & no-stratum             &  fixed                             & 85.05 $\pm$ 0.35          \\
                               & ASHA$_{stratum}$ & geometric     & stratum &  fixed                            & 85.35 $\pm$ 0.07          \\
                               & ASHA$_{stratum\_not\_fixed}$    & geometric     & stratum & not-fixed                 & 85.16 $\pm$ 0.07          \\
                               & ACE$_{noskip}$   & linear        & stratum & not-fixed                 & \textbf{85.41 $\pm$ 0.09} \\ \midrule
\multirow{4}{*}{Random search} & ASHA             & geometric     & no-stratum &  fixed                             & 84.19 $\pm$ 0.35          \\
                               & ASHA$_{stratum}$ & geometric     & stratum & fixed                              & 85.27 $\pm$ 0.16          \\
                               & ASHA$_{stratum\_not\_fixed}$    & geometric     & stratum & not-fixed                    & 85.32 $\pm$ 0.17          \\
                               & ACE$_{noskip}$   & linear        & stratum & not-fixed                & \textbf{85.37 $\pm$ 0.02} \\ \bottomrule
\end{tabular}
}
\label{tab:ablation_geometric}
\end{table}

\begin{table}[h]
\caption{twenty-one test tasks and examples}
\centering
\scalebox{0.75}{

\begin{tabular}{@{}lll@{}}
\toprule
\# & Test task name                   & Example                                                                                                                                                                                                   \\ \midrule
1  & Single positive words            & fun                                                                                                                                                                                                       \\ \midrule
2  & Single negative words            & boring                                                                                                                                                                                                    \\ \midrule
3  & Single neutral words             & saw                                                                                                                                                                                                       \\ \midrule
4  & Sentiment-laden words in context & That staff is boring.                                                                                                                                                                                     \\ \midrule
5  & Neutral words in context         & I see this aircraft.                                                                                                                                                                                      \\ \midrule
6  & Intensifiers                     & This is an incredibly unpleasant service.                                                                                                                                                                 \\ \midrule
7  & Change neutral words with BERT   & \begin{tabular}[c]{@{}l@{}}@USAirways my in-laws just Cancelled Flighted 4 tonight. \\ U auto rebooked 4 just on Tuesday that doesn't work. \\ Can you help reFlight Booking Problems them?\end{tabular}  \\ \midrule
8  & Add positive phrases             & \begin{tabular}[c]{@{}l@{}}@JetBlue so technically I could drive to JFK now and put in. \\ Request for tomorrow's flight. You are sweet.\end{tabular}                                                     \\ \midrule
9  & Add negative phrases             & \begin{tabular}[c]{@{}l@{}}@JetBlue we are well aware. Insufficient info. \\ No options. You are poor.\end{tabular}                                                                                       \\ \midrule
10 & Add random urls and handles      & \begin{tabular}[c]{@{}l@{}}@united My flying United is over...sorry. \\ The Captain still had 20 minutes of pre-flight preparations \\ to make while we sat with no air! https://t.co/befys3\end{tabular} \\ \midrule
11 & Punctuation                      & \begin{tabular}[c]{@{}l@{}}@united pleasantly surprised with quality of service and flight.\\ Flew LGA-CLE-DEN. Friendly crew. Love the concept of \#byod \#worksnicely.\end{tabular}                     \\ \midrule
12 & Typos                            & \begin{tabular}[c]{@{}l@{}}@VirginAmerica - can you tweet me the Cancelled Flight/chng fee for a flight? \\ or can I rebook nuder one of your affiliates? If so, who are afiliates?\end{tabular}          \\ \midrule
13 & 2 typos                          & \begin{tabular}[c]{@{}l@{}}@united @annricord 0162431184663.\textbackslash{}n3 o fyour agents said we wouldb e refunded. \\ Agents said United should never have sold us the ticket.\end{tabular}         \\ \midrule
14 & Contractions                     & \begin{tabular}[c]{@{}l@{}}@united didn't get her name. She was not in our group. She was sitting behind us. \\ Think it was window seat \#40? We only overheard...\end{tabular}                          \\ \midrule
15 & Change names                     & \begin{tabular}[c]{@{}l@{}}@AmericanAir You guys did an amazing job today! Know it’s hard; \\ thanks to Kate Appleton for all her hard work reFlight \\ Booking Problems my friends and me!\end{tabular}  \\ \midrule
16 & Change locations                 & \begin{tabular}[c]{@{}l@{}}@SouthwestAir Gate attendant at McCarran C16 (Vegas to Dallas) went above and beyond. \\ After a long day of frustration it was welcome.\end{tabular}                          \\ \midrule
17 & Change numbers                   & \begin{tabular}[c]{@{}l@{}}@SouthwestAir Your onboard wifi is so bad it's taking me 20 minutes to send this tweet. \\ Working is off the table. \#disappointed\end{tabular}                               \\ \midrule
18 & Protected race                   & Melanie is a black migrant.                                                                                                                                                                               \\ \midrule
19 & Protected sexual                 & Jesse is an asexual father.                                                                                                                                                                               \\ \midrule
20 & Protected religion               & Ryan is a Christian student.                                                                                                                                                                              \\ \midrule
21 & Protected nationality            & Destiny is a Chinese developer.                                                                                                                                                                           \\ \bottomrule
\end{tabular}

}
\label{tab:test_tasks}
\end{table}

\end{document}